\pdfoutput=1  

\documentclass[11pt]{article}

\usepackage[letterpaper,margin=1in]{geometry}

\usepackage[T1]{fontenc}
\usepackage{graphicx}
\usepackage{amsmath,amssymb}
\usepackage{amsthm}
\usepackage{bm}
\usepackage{booktabs}
\usepackage{multirow}
\usepackage{algorithm}
\usepackage{algpseudocode}
\usepackage{microtype}
\usepackage{caption}
\usepackage{subcaption}
\usepackage{xcolor}
\usepackage{enumitem}
\usepackage{cite}
\definecolor{linkblue}{rgb}{0.12,0.49,0.85}
\usepackage[colorlinks=true,linkcolor=linkblue,citecolor=linkblue,
            urlcolor=linkblue,breaklinks=true]{hyperref}
\usepackage[accsupp]{axessibility}   

\newcommand{\JT}{J^{\!\top}}
\newcommand{\LT}{L^{\!\top}}
\newcommand{\WT}{W^{\!\top}}
\newcommand{\sigp}{\sigma^{\prime}}
\newcommand{\eps}{\boldsymbol{\varepsilon}}
\newcommand{\norm}[1]{\left\lVert#1\right\rVert}

\theoremstyle{plain}
\newtheorem{observation}{Observation}
\newtheorem{proposition}{Proposition}
\newtheorem{corollary}{Corollary}
\theoremstyle{remark}
\newtheorem{remark}{Remark}

\begin{document}

\title{Weight Feedback Computes the Jacobian Transpose Locally\\
in Modern Deep Networks}

\author{
  Junlong Shen\textsuperscript{1,2} \qquad
  Xingyu Li\textsuperscript{1,2}\thanks{Corresponding author:
    \texttt{xingyu@ualberta.ca}.}\\[4pt]
  \normalsize\textsuperscript{1}Department of Electrical and Computer Engineering,
  University of Alberta, Edmonton, Canada\\
  \normalsize\textsuperscript{2}Alberta Machine Intelligence Institute (Amii),
  Edmonton, Canada\\[2pt]
  \normalsize\texttt{junlong6@ualberta.ca, xingyu@ualberta.ca}
}
\date{}

\maketitle

\begingroup
\renewcommand\thefootnote{}%
\footnotetext{Preprint. Accepted to ECCV~2026; the definitive version will appear
in the conference proceedings.}%
\endgroup

\begin{abstract}
Predictive Coding (PC) offers a biologically motivated alternative to
backpropagation via local weight updates, yet routing error between layers
still relies on an autograd Jacobian-transpose ($\JT$) product---the last
non-local operation in PC.
We show that this dependency is largely avoidable.
For any layer $f(x){=}\mathrm{Act}(\mathrm{Norm}(L(x)))$ with frozen
normalization statistics, the exact $\JT$ \emph{factors} into three locally
available terms,
$\JT v = \LT(\mathbf{s} \odot \sigp(\mathrm{pre\text{-}act}) \odot v)$,
where $\sigp$ is the activation derivative and
$\mathbf{s}{=}\gamma/\sigma_{\mathrm{run}}$ is the normalization gain.
Prior weight-feedback methods omitted both corrections; restoring them closes
the transport gap for this layer class.
Note that locality here holds \emph{up to} three
assumptions, which we state upfront---weight symmetry ($\LT$ mirrors the
forward operator, as assumed by all PC), a soft spectral-norm control that is
not synapse-local, and a nearest-neighbour approximation for MaxPool.
Substituting the identity into PC yields WF-Act-PC, which removes the autograd
backward pass from error transport.
On CIFAR-10/100 (50 epochs, 5 seeds), WF-Act-PC is the only PC method whose
accuracy \emph{improves} with depth, surpassing iPC---the strongest classical
PC baseline---by 2.7--22.3\,pp on CIFAR-10.
With both methods tuned per architecture, it matches or exceeds a
comparably-tuned backpropagation baseline on the deeper CIFAR-10 architectures
(VGG-9: 93.57\% vs.\ 92.43\%; ResNet-18: 92.76\% vs.\ 91.54\%) and on the
harder Tiny-ImageNet benchmark, while trailing tuned BP on the deeper
CIFAR-100 VGG cells.
Our WF-Act-PC implementation is publicly available at
\url{https://github.com/jlshen025/pcax}.

\end{abstract}

\section{Introduction}
\label{sec:intro}

Backpropagation~\cite{rumelhart1986learning} computes weight gradients by
propagating error signals through the Jacobian transpose~$\JT$ of each layer.
Two properties of this procedure lack established neural
correlates~\cite{crick1989recent}: feedback synaptic weights must mirror
feedforward weights exactly (\emph{weight symmetry}), and error signals must
propagate globally before any weight updates (\emph{update locking}).
Predictive Coding (PC)~\cite{rao1999predictive,friston2005theory} resolves
update locking through local weight updates at each
layer~\cite{whittington2017approximation,salvatori2024stable}, making it
among the most promising bio-plausible learning frameworks.
Yet computing $\JT$ during error transport---the operation that routes
credit between layers---still requires autograd vector-Jacobian products
(VJPs)~\cite{millidge2022predictive}: the last non-local operation in PC.
Feedback Alignment (FA)~\cite{lillicrap2016random} takes the opposite
approach, replacing $\JT$ with a fixed random matrix $B$, but fails on deep
convolutional networks~\cite{bartunov2018assessing,launay2020direct}.
Computing exact error signals without a global backward pass---the
$\JT$ transport problem---has accordingly been regarded as an inherent
limitation of weight-feedback architectures.

We revisit this view.
For any layer of the form $f(x) = \mathrm{Act}(\mathrm{Norm}(L(x)))$---the
standard building block of modern deep networks---weight feedback augmented
with two locally computable corrections recovers the \emph{exact} $\JT$
(Fig.~\ref{fig:method}):
\begin{equation}
  \JT \cdot v \;=\;
    \LT\!\bigl(\,\mathbf{s} \odot
    \sigp(\mathrm{pre\text{-}act}) \odot v\,\bigr),
  \label{eq:main_result}
\end{equation}
where $\LT$ is the transpose of the linear operator (the symmetric synapse),
$\sigp(\mathrm{pre\text{-}act})$ is the pointwise activation derivative
(analogous to an STDP eligibility trace), and
$\mathbf{s} = \gamma/\sigma_{\mathrm{run}}$ is the per-channel normalization
gain (analogous to homeostatic divisive
normalization~\cite{carandini2012normalization}).
The result is elementary---three applications of the chain rule
(Observation~\ref{thm:wf}); what matters is the \emph{locality} of the factors,
not the derivation.
The key point is that frozen normalization---running normalization layers
in eval mode during the inner loop, a natural choice for inference-loop
stability---collapses
$J_{\mathrm{Norm}}$ to a fixed diagonal, reducing the full $\JT$ to
purely local factors.

\medskip\noindent\textbf{Scope and assumptions.}
To clarify, the proposed method makes precise what ``local'' does and does
not mean here.
WF-Act-PC computes $\JT$ from local factors \emph{up to} three assumptions:
(i)~\emph{weight symmetry}---$\LT$ is taken to mirror the feedforward operator
$L$, the same assumption PC (and BP) already make; (ii)~a \emph{soft
spectral-norm control} applied to weights for
stability, which uses power iteration and is therefore not synapse-local
(though it runs on the slow learning timescale, and disabling it costs only
$0.85$\,pp; Sect.~\ref{sec:ablation}); and (iii)~a \emph{nearest-neighbour
approximation for MaxPool}, whose exact $\JT$ would require cached argmax
indices (the approximation costs ${\approx}1.25$\,pp; Sect.~\ref{sec:ablation}).
Within these bounds, Eq.~(\ref{eq:main_result}) is exact for the non-pooling,
frozen-normalization portion of the network.

The identity was overlooked for two historical reasons.
Classical FA~\cite{lillicrap2016random} predated
BatchNorm~\cite{ioffe2015batch}: FA analysed $f = \sigma(Wx)$, where
$\WT v$ is genuinely only an approximation---missing $\sigp$.
When normalization became standard, the FA community moved to learned feedback
matrices~\cite{akrout2019deep,kunin2020two} and target
propagation~\cite{lee2015difference}, treating the transport gap as inherent.
PC, meanwhile, commonly freezes batch statistics during inner-loop
inference (running BatchNorm in eval mode) but viewed this as an engineering
workaround; the implication---that $J_{\mathrm{Norm}}$ collapses to a fixed
diagonal, rendering the remaining correction purely local---went unnoticed.
The two communities each held one piece:
FA knew $\WT$ was approximate but lacked frozen normalization;
PC had frozen normalization but never examined its Jacobian structure.

Substituting Eq.~(\ref{eq:main_result}) into the PC error-transport step
yields WF-Act-PC, which removes the global backward pass from error
transport---local up to the weight-symmetry and soft spectral-norm caveats
above---with no auxiliary feedback networks and no target propagation.
On CIFAR-10 (cross-entropy, 50 epochs, 5 seeds), WF-Act-PC surpasses
iPC~\cite{salvatori2024stable}---the strongest classical PC baseline---on
every architecture by $2.7$--$22.3$\,pp, and, with both methods tuned per
architecture, reaches BP-level accuracy on deep architectures
(VGG-9: $93.57$\% vs.\ $92.43$\%;
ResNet-18: $92.76$\% vs.\ $91.54$\%).
On CIFAR-100, WF-Act-PC scales from $65.64$\% (VGG-5) to $71.07$\%
(ResNet-18), exceeding tuned BP by $+0.85$\,pp on ResNet-18, while
iPC collapses ($56.07 \to 29.45$\%) and DPC-CN degrades
($66.85 \to 60.59$\%); on the deepest CIFAR-100 VGG cells, however, tuned BP
catches up (a tie on VGG-7 and a $0.84$\,pp BP lead on VGG-9).
On Tiny-ImageNet (200 classes), WF-Act-PC again improves with depth
($46.0 \to 61.1$\%, VGG-5$\to$ResNet-18) and exceeds tuned BP at both depths.

\medskip\noindent\textbf{Contributions.}
\begin{enumerate}[leftmargin=*,topsep=2pt,itemsep=1pt]
\item \textbf{Theoretical.}
  We observe that, under frozen normalization, the exact Jacobian transpose
  of any $\mathrm{Act}(\mathrm{Norm}(L(\cdot)))$ layer factors into locally
  available terms---$\sigp$ and $\mathbf{s}$ alongside $\LT$
  (Observation~\ref{thm:wf}). This \emph{locality} resolves the $\JT$ transport
  problem for the layer class underlying modern deep networks.
\item \textbf{Algorithmic.}
  We derive WF-Act-PC, a PC algorithm for deep convolutional networks in
  which error transport is computed from local factors---local up to weight
  symmetry and a soft, slow-timescale spectral-norm control, with no autograd
  VJP in the inner loop.
\item \textbf{Empirical.}
  WF-Act-PC surpasses iPC, the strongest classical PC baseline, by
  $2.7$--$22.3$\,pp and is the only PC method whose accuracy improves with
  depth, on CIFAR-10/100 and Tiny-ImageNet.
  With both methods tuned per architecture, it reaches BP-level accuracy on
  deep architectures (VGG-7/9, ResNet-18), narrowing the accuracy gap between
  PC and backpropagation on deep convolutional networks.
\end{enumerate}

\begin{figure}[t]
  \centering
  \includegraphics[width=\linewidth]{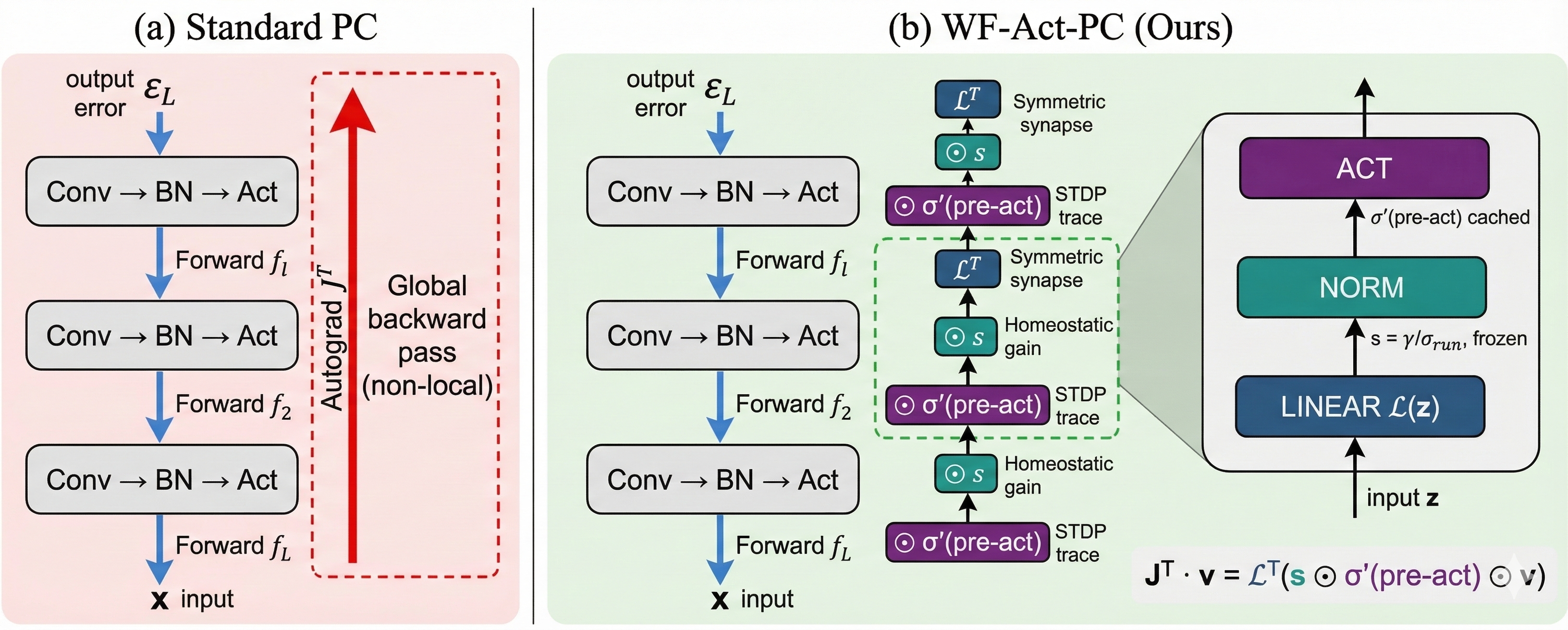}
  \caption{\textbf{WF-Act-PC overview.}
  Each layer $f_l = \mathrm{Act}(\mathrm{Norm}(L(\cdot)))$ maintains three
  local factors: the symmetric synapse $\LT$, the activation derivative
  $\sigp$ (STDP trace), and the normalization gain
  $\mathbf{s} = \gamma/\sigma_{\mathrm{run}}$ (homeostatic gain).
  Error transport uses
  $\JT v = \LT(\mathbf{s} \odot \sigp \odot v)$; no autograd backward
  pass is needed.
  The crossed-out arrow indicates the eliminated global
  $\JT$ computation.}
  \label{fig:method}
\end{figure}

\section{Related Work}
\label{sec:related}

\paragraph{Predictive Coding.}
PC~\cite{rao1999predictive,friston2005theory} posits that the brain minimizes
hierarchical prediction errors.
Whittington and Bogacz~\cite{whittington2017approximation} showed that PC
weight updates approximate backpropagation gradients at convergence.
Incremental PC (iPC)~\cite{salvatori2024stable} improves convergence by
updating weights simultaneously with states and is the strongest classical PC
variant.
The PCX benchmark~\cite{pinchetti2025benchmarking} provides standardized
evaluation: iPC reaches $85.51$\% on VGG-5/CIFAR-10 but degrades to
$70.44$\% on ResNet-18---a depth-scaling failure shared by all PC
methods tested.

\paragraph{Depth scalability in PC.}
Several recent works address depth scaling from different angles.
Qi et al.~\cite{qi2025training} diagnose \emph{energy imbalance}---prediction
error concentrating near the output---and propose exponentially decayed
weighting (DPC-CN), achieving $89.46$\% on VGG-5 but degrading to $86.55$\%
on VGG-9.
Goemaere et al.~\cite{goemaere2025epc} identify exponential signal decay in
state-based PC and introduce error-based PC (ePC), a reparametrization that
eliminates signal decay and matches backpropagation on deep architectures.
Innocenti et al.~\cite{innocenti2025mupc} scale PC to 100+ layers via
Depth-$\mu$P parameterization ($\mu$PC), but only on \emph{fully connected}
residual networks; the authors note that CIFAR-10 performance is ``far
from SOTA because of the fully connected (as opposed to convolutional)
architectures used,'' and no convolutional results are reported.
All three approaches improve depth robustness but \emph{retain autograd
$\JT$ computation}---the non-local operation that our observation
addresses---and none demonstrate competitive accuracy on standard
convolutional benchmarks.

\paragraph{Feedback alignment and the weight transport problem.}
FA~\cite{lillicrap2016random} replaces $\JT$ with a fixed random matrix $B$;
Direct FA~\cite{nokland2016direct} projects output errors directly to each
layer.
Both fail on deep CNNs~\cite{bartunov2018assessing,launay2020direct}.
Observation~\ref{thm:wf} offers a sharper diagnosis: for
$\mathrm{Act}(\mathrm{Norm}(L(\cdot)))$ layers, $\LT v$ is missing exactly
$\sigp$ and $\mathbf{s}$; restoring them makes weight feedback exact (given
the symmetric $\LT$ that FA was designed to avoid).
FA underperformed here not because weight feedback is fundamentally limited,
but because two local corrections were absent.

\paragraph{Other bio-plausible methods.}
Target propagation~\cite{lee2015difference} learns separate inverse mappings;
learned FA~\cite{akrout2019deep,kunin2020two} trains feedback matrices toward
$\WT$.
Both introduce auxiliary parameters and training complexity.
WF-Act-PC avoids these: given symmetric $\WT$, two local corrections suffice
for exact transport on this layer class.

\section{Preliminaries}
\label{sec:prelim}

Consider a feedforward network with mappings $\{f_l\}_{l=1}^{L}$,
parameters $\{W_l\}$, input $x$, and target $y$.
PC defines layer-wise prediction errors
$\eps_l = \mu_l - f_l(\mu_{l-1})$ and the variational free energy
\begin{equation}
  E = \frac{1}{2}\sum_{l=1}^{L-1} \norm{\eps_l}^2
      + \mathcal{L}\!\bigl(f_L(\mu_{L-1}),\, y\bigr),
  \label{eq:pc_energy}
\end{equation}
where $\mathcal{L}$ is the output loss (cross-entropy in all our experiments).
Training alternates between \emph{inference}---minimizing $E$ w.r.t.\ latent
states $\{\mu_l\}$ (or equivalently, errors $\{\eps_l\}$)---and
\emph{learning}---updating weights $\{W_l\}$ at the equilibrated states.

\paragraph{Error transport and the $\JT$ problem.}
Under a top-to-bottom (Gauss-Seidel) update
schedule~\cite{millidge2022predictive}, the energy gradient w.r.t.\
$\eps_l$ takes the form
\begin{equation}
  \frac{\partial E}{\partial \eps_l}
    = \eps_l - J_{l+1}^{\top} \cdot \eps_{l+1},
  \label{eq:pc_grad}
\end{equation}
where $J_{l+1}$ is the Jacobian of $f_{l+1}$ evaluated at the current state.
Weight updates are local---$\partial E / \partial W_l$ depends only on
activities at layers $l$ and $l{-}1$.
The \emph{sole non-local computation} is the
$J_{l+1}^{\top} \cdot \eps_{l+1}$ term: for a typical Conv-BN-Act block,
evaluating this product requires an autograd backward pass through the entire
block.

\paragraph{Error-based formulation (from ePC).}
Following Goemaere et al.~\cite{goemaere2025epc}, we optimize
directly over errors $\eps_l$ rather than states~$\mu_l$, with states
recovered as $\mu_l = f_l(\mu_{l-1}) + \eps_l$.
This reparametrization eliminates exponential signal
decay during inference: errors at all layers receive gradient signals from
the first iteration, rather than waiting for signals to propagate
layer by layer.
The energy and gradient structure
(Eqs.~\ref{eq:pc_energy}--\ref{eq:pc_grad}) remain identical; only the
optimization variables change.
WF-Act-PC adopts two components from ePC---error-based variables and a
top-to-bottom sequential (Gauss-Seidel) sweep---and pairs them with an
RMSProp inner-loop optimizer.
Observation~\ref{thm:wf} replaces the remaining autograd $\JT$ computation
with a closed-form local formula, making the inner loop free of autograd
VJPs.
The two advances address distinct bottlenecks: ePC resolves signal decay;
our observation eliminates non-local transport.

\section{Method}
\label{sec:method}

\subsection{Local Factorization of the Jacobian Transpose}
\label{sec:theorem}

Weight feedback is conventionally treated as an approximation of $\JT$.
We make a simple observation: for modern normalized layers, the exact
$\JT$---the same quantity backpropagation computes---factors into terms that
are each locally available, once two corrections are included.
We record this as Observation~\ref{thm:wf}.

\begin{observation}[Local factorization of the Jacobian transpose]
\label{thm:wf}
Let $f(x) = \mathrm{Act}(\mathrm{Norm}(L(x)))$ where:
\begin{itemize}[topsep=2pt,itemsep=1pt]
  \item $L$ is any linear operator (convolution, fully-connected, depthwise
    convolution, \ldots) with adjoint $\LT$;
  \item $\mathrm{Norm}$ is any affine normalization with frozen per-channel
    scale $\mathbf{s}$ and shift $\mathbf{b}$, so that
    $\mathrm{Norm}(z) = \mathbf{s} \odot z + \mathbf{b}$ with $\mathbf{s}$
    fixed during error propagation
    (BatchNorm~\cite{ioffe2015batch}: $s_c = \gamma_c/\sigma^{\mathrm{run}}_c$;
     LayerNorm~\cite{ba2016layer}: $s = \gamma/\sigma_{\mathrm{LN}}$;
     GroupNorm~\cite{wu2018group}: $s = \gamma_g/\sigma_G$;
     \textnormal{no norm}: $\mathbf{s} = \mathbf{1}$);
  \item $\mathrm{Act}$ is any pointwise activation with derivative $\sigp$.
\end{itemize}
Then for any vector $v$ in the output space of $f$:
\begin{equation}
  \JT \cdot v \;=\;
    \LT\!\Bigl(\,\mathbf{s} \odot
    \sigp\!\bigl(\mathrm{Norm}(L(x))\bigr) \odot v\,\Bigr).
  \label{eq:wf_identity}
\end{equation}
\end{observation}

\begin{proof}
By the chain rule, $J = J_{\mathrm{Act}} \cdot J_{\mathrm{Norm}} \cdot J_L$,
so $\JT = J_L^{\top} \cdot J_{\mathrm{Norm}}^{\top} \cdot
J_{\mathrm{Act}}^{\top}$.
We evaluate each factor.

\textbf{Step 1} (activation).
Since $\mathrm{Act}$ is pointwise,
$J_{\mathrm{Act}} = \mathrm{diag}(\sigp(\mathrm{pre\text{-}act}))$, giving
$J_{\mathrm{Act}}^{\top} \cdot v = \sigp(\mathrm{pre\text{-}act}) \odot v$.

\textbf{Step 2} (frozen normalization).
Because $\mathrm{Norm}(z) = \mathbf{s} \odot z + \mathbf{b}$ with
$\mathbf{s}$ frozen, $J_{\mathrm{Norm}} = \mathrm{diag}(\mathbf{s})$, giving
$J_{\mathrm{Norm}}^{\top} \cdot u = \mathbf{s} \odot u$.

\textbf{Step 3} (linear operator).
By definition of the adjoint, $J_L^{\top} = \LT$.

\textbf{Composition.}
\[
  \JT \cdot v
  = \LT\!\bigl(J_{\mathrm{Norm}}^{\top}(J_{\mathrm{Act}}^{\top} \cdot v)\bigr)
  = \LT\!\bigl(\mathbf{s} \odot \sigp(\mathrm{pre\text{-}act}) \odot v\bigr).
  \qedhere
\]
\end{proof}

\noindent
The derivation is elementary; what had gone unrecognised in prior
weight-transport research is that both corrections are \emph{locally
available} (see Sect.~\ref{sec:intro} for the historical analysis).
A dimension-tracking proof for the Conv+BN+Act case and instantiations for
other normalization schemes appear in the supplementary material.

\begin{remark}[Frozen normalization is a feature, not a constraint]
\label{rem:frozen}
Three independent motivations converge on frozen statistics.
\emph{Mathematically}, frozen $\mathbf{s}$ is required for
$J_{\mathrm{Norm}}$ to collapse to a fixed diagonal (Step~2).
\emph{Algorithmically}, running BatchNorm in eval mode during the inner
loop prevents corruption of running statistics, a natural choice for
inference-loop stability.
\emph{Biologically}, homeostatic normalization gains are slow variables,
stable across a single perception
episode~\cite{carandini2012normalization}.
The convergence of all three motivations on the same condition is not
coincidental: the same frozen statistics make PC biologically sensible and
weight feedback exact.
\end{remark}

\begin{remark}[Special cases]
\label{rem:special}
Without normalization ($\mathbf{s} = \mathbf{1}$),
Eq.~(\ref{eq:wf_identity}) reduces to
$\JT v = \LT(\sigp(L(x)) \odot v)$---the textbook backprop formula;
classical FA further drops $\sigp$.
For MaxPool, our implementation uses nearest-neighbour upsampling
(repeating each value over the $2{\times}2$ pool window), a local
approximation that avoids caching argmax positions and is the one place
Eq.~(\ref{eq:wf_identity}) does not hold exactly.
The autograd-vs-local ablation (Table~\ref{tab:ablation}) confirms this
approximation incurs small accuracy loss.
\end{remark}

\begin{remark}[Residual skip connections]
\label{rem:residual}
For $f(x) = f_{\mathrm{plain}}(x) + x$, we have
$\JT v = J_{\mathrm{plain}}^{\top} v + v$;
Observation~\ref{thm:wf} covers $J_{\mathrm{plain}}^{\top}$ exactly.
The identity bypass ensures $\sigma_{\max}(J) \geq 1$, so spectral norm
clipping alone cannot guarantee contraction.
Top-layer gradient clipping ($\norm{v_{\mathrm{top}}} \leq c$) bounds
errors before the Gauss-Seidel loop, enabling stable training with
skip scale $= 1.0$.
\end{remark}

\begin{corollary}[Local computability]
\label{cor:bio}
Every factor in Eq.~(\ref{eq:wf_identity}) is locally computable:

\begin{center}
\renewcommand{\arraystretch}{1.2}
\begin{tabular}{lll}
  \toprule
  \textbf{Factor} & \textbf{Computation} & \textbf{Neural mechanism} \\
  \midrule
  $\LT$ & Adjoint of feedforward operator & Symmetric synapse \\
  $\sigp(\mathrm{pre\text{-}act})$ & Activation derivative at pre-act
    & STDP eligibility trace \\
  $\mathbf{s} = \gamma/\sigma_{\mathrm{run}}$ & Per-channel norm.\ gain
    & Homeostatic gain control \\
  \bottomrule
\end{tabular}
\end{center}

\noindent
No global backward pass or separate feedback weight matrices are required.
The symmetric synapse $\LT$, however, is the same $W^{\top}$ already assumed
by the PC framework: WF-Act-PC makes this pre-existing assumption locally
computable but does \emph{not} remove it---a limitation we return to in
Sect.~\ref{sec:conclusion}.
\end{corollary}

\subsection{The WF-Act-PC Algorithm}
\label{sec:algorithm}

Substituting Observation~\ref{thm:wf} into the PC error-transport
step~(\ref{eq:pc_grad}) yields a local energy gradient for each
layer~$l$:
\begin{equation}
  \frac{\partial E}{\partial \eps_l}
  = \eps_l - L_{l+1}^{\top}\!\bigl(\,\mathbf{s}_{l+1} \odot
    \sigp(\mathrm{pre\text{-}act}_{l+1}) \odot \eps_{l+1}\,\bigr).
  \label{eq:wfactpc_update}
\end{equation}
All quantities are locally available:
$\eps_l$ and $\eps_{l+1}$ are adjacent-layer errors;
$\mathbf{s}_{l+1}$ is the normalization gain
($\gamma_{l+1}/\sigma^{\mathrm{run}}_{l+1}$ for BatchNorm,
$\bm{1}$ without normalization);
$\sigp(\mathrm{pre\text{-}act}_{l+1})$ is cached from the forward pass; and
$L_{l+1}^{\top}$ is transposed convolution for convolutional layers or
$W^{\top}$ for dense layers.
We minimize $E$ w.r.t.\ $\{\eps_l\}$ via RMSProp gradient descent on
Eq.~(\ref{eq:wfactpc_update}), with a top-to-bottom Gauss-Seidel sweep
exploiting the triangular dependency structure.
The full procedure is given in Algorithm~\ref{alg:wfactpc}.

\begin{algorithm}[t]
\caption{WF-Act-PC --- one training step}
\label{alg:wfactpc}
\begin{algorithmic}[1]
\Require Batch $(x, y)$; model $\{f_l, W_l\}$; iterations $T$;
         inference step $\eta_{\varepsilon}$; RMSProp decay $\alpha_{\mathrm{rms}}$;
         weight step $\eta_w$; SNC threshold $\tau$; error clip $c$
\State \textbf{Forward pass:} compute $h_l = f_l(h_{l-1})$, $h_0 = x$;
       cache $\mathrm{pre\text{-}act}_l$ and $\mathbf{s}_l$ for each $l$
       ($\mathbf{s}_l = \bm{1}$ if no normalization)
\State $\eps_l \leftarrow \bm{0}$, \; $r_l \leftarrow \bm{1}$ for all $l$
       \Comment{Errors and RMSProp states}
\State $\eps_L \leftarrow \tilde{y} - \mathrm{softmax}(h_L)$; \quad
       $\eps_L \leftarrow \eps_L \cdot \min\!\bigl(1,\;
       c\, / \norm{\eps_L}\bigr)$
       \Comment{Output error + clip}
\For{$t = 1$ to $T$} \Comment{Gauss-Seidel top $\to$ bottom}
  \For{$l = L{-}1$ \textbf{downto} $1$}
    \State $g_l \leftarrow \eps_l - L_{l+1}^{\top}\!\bigl(\mathbf{s}_{l+1}
             \odot \sigp(\mathrm{pre\text{-}act}_{l+1})
             \odot \eps_{l+1}\bigr)$
           \Comment{Local gradient~(\ref{eq:wfactpc_update})}
    \State $r_l \leftarrow (1{-}\alpha_{\mathrm{rms}})\,r_l
             + \alpha_{\mathrm{rms}}\,g_l^2$
    \State $\eps_l \leftarrow \eps_l
             - \eta_{\varepsilon}\, g_l \,/\, \sqrt{r_l + 10^{-8}}$
           \Comment{RMSProp update}
  \EndFor
\EndFor
\State Compute $E_{\mathrm{local}}$; update $W_l \leftarrow W_l -
       \eta_w \cdot \mathrm{Adam}(\nabla_{W_l} E_{\mathrm{local}})$
       \Comment{Local weight update}
\For{each layer $l$ with $\sigma_{\max}(W_l) > \tau$} \Comment{Soft SNC}
  \State $W_l \leftarrow W_l -
         (\sigma_{\max}(W_l) - \tau)
         \,\bm{u}_l\bm{v}_l^{\top}$
\EndFor
\end{algorithmic}
\end{algorithm}

We monitor convergence via the \emph{convergence gap}
\begin{equation}
  \Delta = E_{\mathrm{local}} - \tfrac{1}{2}\textstyle\sum_l \norm{\eps_l}^2
         \;\geq\; 0.
  \label{eq:gap}
\end{equation}
When $\Delta \approx 0$, the inner loop has converged and weight updates are
exact local gradients; when $\Delta \gg 0$, updates are noisy.
This provides an online diagnostic of training health without additional cost.

\subsection{Soft Spectral Norm Clipping}
\label{sec:snc}

Stable training requires controlling the spectral norm of each layer's
weight matrix.
After each weight update, if $\sigma_{\max}(W_l) > \tau$, we apply a soft
rank-1 correction:
\begin{equation}
  W_l \;\leftarrow\; W_l
    - \bigl(\sigma_{\max}(W_l) - \tau\bigr)\,\bm{u}_l\,\bm{v}_l^{\top},
  \label{eq:soft_snc}
\end{equation}
where $(\bm{u}_l, \bm{v}_l)$ are the leading singular vectors of $W_l$
(estimated by $n_{\mathrm{steps}} = 20$ power iterations).
Unlike hard clipping ($W \leftarrow W \cdot \tau/\sigma_{\max}$), this
modifies only the leading singular direction, avoiding catastrophic weight
rescaling during transient spikes.

Power iteration is \emph{not} synapse-local---together with weight symmetry,
one of the two assumptions under which WF-Act-PC is ``local''.
It operates on weights (not error signals) and is decoupled from the
per-sample inference loop---analogous to homeostatic synaptic
scaling~\cite{turrigiano2004homeostatic} on the slow learning timescale.
Disabling SNC reduces accuracy by only $0.85$\,pp
(Table~\ref{tab:ablation}), so WF-Act-PC's core
advantage---local $\JT$ transport---does not depend on it.

Independent of SNC, local error transport has engineering benefits: no
global backward pass, layer-parallel weight updates, $O(1)$ rather than
$O(L)$ activation memory per layer, and a natural mapping to neuromorphic
hardware.

\subsection{Convergence}
\label{sec:convergence}

\begin{proposition}
\label{prop:convergence}
With frozen normalization (and, without BatchNorm, this holds directly),
the PC energy $E$ is quadratic in $\{\eps_l\}$.
For fixed states $\{h_l\}$, the top-to-bottom Gauss-Seidel sweep exploits
the triangular dependency structure: each $\eps_l$ update depends only on
the already-updated $\eps_{l+1}$.
When $\sigma_{\max}(J_{l+1}) < 1$ for all $l$, a single sweep converges
to the global minimiser of $E$ w.r.t.\ $\{\eps_l\}$.
\end{proposition}

\noindent
In practice, the spectral condition is not enforced directly.
Soft SNC maintains $\sigma_{\max}(W_l) \leq \tau$, an empirically sufficient
proxy: with $\tau = 3.0$ and $T{=}20$ RMSProp iterations, the convergence gap
$\Delta$ reaches near-zero on all architectures
(Sect.~\ref{sec:ablation}).
For residual architectures, gradient clipping of $v_{\mathrm{top}}$ provides
additional stabilization (Remark~\ref{rem:residual}).

\section{Experiments}
\label{sec:experiments}

\subsection{Setup}
\label{sec:setup}

\paragraph{Datasets.}
CIFAR-10 and CIFAR-100~\cite{krizhevsky2009learning}: 50{,}000 training /
10{,}000 test images across 10 and 100 classes.
Both use RandomCrop(32, padding=4) and RandomHorizontalFlip(0.5), matching
the PCX benchmark protocol~\cite{pinchetti2025benchmarking}.
We additionally evaluate on Tiny-ImageNet (200 classes, $64{\times}64$) to
test depth scaling beyond CIFAR (Sect.~\ref{sec:cifar100}).

\paragraph{Architectures.}
VGG-5/7/9~\cite{simonyan2015very} (channel configs [128, 256, 512, 512],
[128, 128, 256, 256, 512, 512],
[128, 128, 256, 256, 512, 512, 512, 512])
and standard ResNet-18~\cite{he2016deep},
following PCX benchmark configurations.
All four architectures use Conv+GELU blocks \emph{without} BatchNorm
($\mathbf{s} = \bm{1}$ in Observation~\ref{thm:wf}): the VGGs use
Conv+GELU+MaxPool, and ResNet-18 adds residual connections.
The main experiments therefore exercise the $\sigp$ correction---and, for
ResNet-18, the residual handling of Remark~\ref{rem:residual}.
We validate the remaining normalization-gain correction
($\mathbf{s} = \gamma/\sigma_{\mathrm{run}} \neq \bm{1}$) separately by adding
BatchNorm to VGG-5 (Sect.~\ref{sec:ablation}, Table~\ref{tab:ablation_bn}).

\paragraph{Hyperparameters.}
All CIFAR-10 experiments share: Soft SNC ($\tau = 3.0$,
$n_{\mathrm{steps}} = 20$), warmup-cosine LR schedule (2 epoch warmup,
decay to $10^{-6}$), AdamW ($\beta_1{=}0.9$, $\beta_2{=}0.999$, weight
decay $10^{-4}$), RMSProp inner loop ($\alpha_{\mathrm{rms}}{=}0.2$,
$\eta_{\varepsilon}{=}0.1$), $T{=}20$, label smoothing $0.05$,
Cutout (patch size $8$), gradient clipping (norm $\leq 1.0$),
top-layer error clipping $\norm{v_{\mathrm{top}}} \leq 5.0$,
batch size 128, and learning rate $\eta_w = 7{\times}10^{-4}$ with
\emph{no per-model tuning}.

\paragraph{Baselines.}
Prior PC and BP numbers are from the PCX
benchmark~\cite{pinchetti2025benchmarking} (50 epochs, 5 seeds, with
augmentation):
the benchmark's two backpropagation references, BP-CE and BP-SE
(standard backpropagation with cross-entropy and squared-error loss,
respectively), trained under its own recipe;
iPC~\cite{salvatori2024stable}, the strongest PC-family method in the
benchmark;
DPC-CN~\cite{qi2025training}, which adds exponential energy weighting;
and ePC~\cite{goemaere2025epc}, which reparametrises PC in error
coordinates.
For a controlled comparison, we additionally report ``BP (tuned)'':
backpropagation trained with the \emph{same} augmentation recipe as
WF-Act-PC (Cutout, label smoothing, cosine warmup-decay) and tuned
per architecture (learning rate searched over
$\{3,7\}{\times}10^{-4}, \{1,3\}{\times}10^{-3}$).
Because WF-Act-PC and BP have different optimal learning rates, tuning each
method at its own best LR is the fair protocol; the full grid is in
Suppl.\ Mat.\ Sect.~H.
Any gap between the benchmark's BP references and BP~(tuned) reflects
differences in loss and training recipe, not the learning algorithm.

\subsection{Main Results: CIFAR-10}
\label{sec:main_results}

\begin{table}[t]
\centering
\caption{Test accuracy (\%) at 50 epochs (CE loss unless noted; BP-SE uses
squared-error loss).
BP-CE, BP-SE, and prior PC baselines: mean\,$\pm$\,std over 5 seeds
from~\cite{pinchetti2025benchmarking,qi2025training,goemaere2025epc}.
WF-Act-PC and BP (tuned): mean\,$\pm$\,std over 5 seeds.
\emph{Bio?}: ``Yes'' = $\JT$ transport is locally computable via
Observation~\ref{thm:wf} (no autograd VJP), up to weight symmetry and soft SNC;
``Partial'' = local weight updates but $\JT$ requires autograd;
``No'' = global backward pass.
\textbf{Bold}: best per column among methods using the controlled
augmentation recipe; BP-CE and BP-SE are the benchmark's own BP references
(cross-entropy and squared-error loss) and are excluded from bolding.}
\label{tab:main}
\renewcommand{\arraystretch}{1.15}
\resizebox{\textwidth}{!}{%
\begin{tabular}{ll@{\hskip 6pt}cccc@{\hskip 10pt}cccc}
  \toprule
  & & \multicolumn{4}{c@{\hskip 10pt}}{\textbf{CIFAR-10}} &
    \multicolumn{4}{c}{\textbf{CIFAR-100}} \\
  \cmidrule(r){3-6} \cmidrule(l){7-10}
  \textbf{Method} & \textbf{Bio?}
    & \textbf{VGG-5} & \textbf{VGG-7} & \textbf{VGG-9}
    & \textbf{ResNet-18}
    & \textbf{VGG-5} & \textbf{VGG-7} & \textbf{VGG-9}
    & \textbf{ResNet-18} \\
  \midrule
  BP-CE~\cite{pinchetti2025benchmarking} (CE)
    & No
    & $88.11_{\pm0.13}$ & $88.60_{\pm0.10}$ & $89.18_{\pm0.08}$
    & $92.83_{\pm0.18}$
    & $60.82_{\pm0.10}$ & $59.96_{\pm0.10}$ & $60.63_{\pm0.28}$
    & $72.32_{\pm0.26}$ \\
  BP-SE~\cite{pinchetti2025benchmarking} (SE)
    & No
    & $89.43_{\pm0.12}$ & $89.91_{\pm0.10}$ & $90.02_{\pm0.18}$
    & $93.21_{\pm0.07}$
    & $66.28_{\pm0.23}$ & $65.36_{\pm0.15}$ & $65.51_{\pm0.23}$
    & $71.89_{\pm0.16}$ \\
  BP (tuned)
    & No
    & $88.03_{\pm0.18}$ & $91.39_{\pm0.05}$ & $92.43_{\pm0.29}$
    & $91.54_{\pm0.36}$
    & $64.09_{\pm0.14}$ & $\mathbf{69.00}_{\pm0.30}$
    & $\mathbf{71.90}_{\pm0.32}$ & $70.22_{\pm0.27}$ \\
  DPC-CN~\cite{qi2025training} (CE)
    & Partial
    & $\mathbf{89.46}_{\pm0.14}$ & $89.11_{\pm0.26}$ & $86.55_{\pm0.81}$
    & ---
    & $\mathbf{66.85}_{\pm0.07}$ & $63.89_{\pm0.22}$ & $60.59_{\pm0.47}$
    & --- \\
  iPC~\cite{salvatori2024stable,pinchetti2025benchmarking} (CE)
    & Partial
    & $85.51_{\pm0.12}$ & $80.15_{\pm0.21}$ & $79.02_{\pm0.21}$
    & $70.44_{\pm0.81}$
    & $56.07_{\pm0.16}$ & $43.99_{\pm0.30}$ & $44.76_{\pm0.40}$
    & $29.45_{\pm1.36}$ \\
  ePC~\cite{goemaere2025epc} (CE)
    & Partial
    & $88.27_{\pm0.18}$ & $88.84_{\pm0.31}$ & $86.81_{\pm0.09}$
    & $91.73_{\pm0.21}$
    & $63.39_{\pm0.25}$ & $58.62_{\pm0.20}$ & $60.65_{\pm0.25}$
    & $69.47_{\pm0.32}$ \\
  \midrule
  \textbf{WF-Act-PC} (ours) (CE)
    & \textbf{Yes}
    & $88.25_{\pm0.31}$ & $\mathbf{92.15}_{\pm0.72}$
    & $\mathbf{93.57}_{\pm0.48}$ & $\mathbf{92.76}_{\pm0.85}$
    & $65.64_{\pm0.57}$ & $68.75_{\pm0.64}$
    & $70.32_{\pm0.53}$ & $\mathbf{71.07}_{\pm0.71}$ \\
  \bottomrule
\end{tabular}}%
\begin{flushleft}
  \footnotesize
  Both WF-Act-PC and BP (tuned) are tuned per architecture (each at its own
  best LR; VGG-5/CIFAR-10 BP not separately retuned---matched-recipe value
  shown). CIFAR-10: AdamW ($\eta_w{=}7{\times}10^{-4}$), Soft SNC
  ($\tau{=}3.0$), $T{=}20$ RMSProp, Cutout(8), 50 epochs, 5 seeds, same
  hyperparameters across all architectures. CIFAR-100: CutMix/Mixup,
  $\eta_w{=}2{\times}10^{-3}$ (VGG-5/7) or $1{\times}10^{-3}$
  (VGG-9, ResNet-18), label smoothing $0.1$; see supplementary.
  The WF-Act-PC VGG-7/9/CIFAR-100 entries above use the main CIFAR-100 recipe
  ($\eta_w{=}2{\times}10^{-3}$/$1{\times}10^{-3}$); under symmetric per-method
  LR tuning (matching BP's protocol, $\eta_w{=}7{\times}10^{-4}$) they reach
  $68.76_{\pm0.31}$ / $71.06_{\pm0.18}$, i.e.\ within $-0.24$ / $-0.84$\,pp of
  tuned BP (Suppl.\ Mat.\ Sect.~H).
\end{flushleft}
\end{table}

Table~\ref{tab:main} presents the main results.
Three patterns emerge.
\emph{First}, WF-Act-PC surpasses iPC---the strongest classical PC
method---on every architecture, with the advantage growing monotonically
with depth: the gap widens from $+2.7$\,pp (VGG-5) to $+22.3$\,pp
(ResNet-18) on CIFAR-10.
Where iPC degrades from $85.51$\% to $70.44$\%, WF-Act-PC
\emph{improves} from $88.25$\% to $93.57$\%.
\emph{Second}, with both methods tuned per architecture,
WF-Act-PC meets or exceeds BP~(tuned) on the deeper CIFAR-10 architectures
($+0.8$, $+1.1$, $+1.2$\,pp on VGG-7/9/ResNet-18); this holds under the
controlled same-recipe comparison---though both of the benchmark's own BP
references exceed WF-Act-PC on ResNet-18 (BP-CE $92.83$\%, BP-SE $93.21$\%
vs.\ $92.76$\%), reflecting the benchmark's stronger per-architecture recipe
(Sect.~\ref{sec:setup}).
On VGG-5 it exceeds same-recipe BP and the benchmark's BP-CE ($88.25$\% vs.\
$88.03$/$88.11$\%) but remains below BP-SE ($89.43$\%) and DPC-CN
($89.46$\%), indicating that the gap to a fully-tuned baseline has not
closed on shallow models.
To our knowledge, this is the first PC method to reach BP-level
accuracy on deep convolutional architectures under a controlled comparison.
\emph{Third}, ePC~\cite{goemaere2025epc}---which shares the error
reparametrisation but uses autograd $\JT$---degrades on deeper VGGs
($88.27 \to 86.81$\%, VGG-5$\to$9), while WF-Act-PC improves;
local transport, not the training recipe alone, drives the
depth-scaling advantage.

\subsection{CIFAR-100, Tiny-ImageNet, and Wall-Clock Time}
\label{sec:cifar100}

\paragraph{CIFAR-100.}
CIFAR-100 presents a harder test of inner-loop stability: with 100 output
classes, error norms are substantially larger than on CIFAR-10.
WF-Act-PC maintains monotonic depth scaling
(Table~\ref{tab:main}, right; Figure~\ref{fig:depth_cifar100}):
$65.64$\% (VGG-5), $68.75$\% (VGG-7), $70.32$\% (VGG-9), $71.07$\%
(ResNet-18).
Prior PC methods either collapse with depth
(iPC: $56.07 \to 29.45$\%, VGG-5$\to$ResNet-18) or degrade on deeper
VGGs (DPC-CN: $66.85 \to 60.59$\%; ePC: $63.39 \to 60.65$\%,
VGG-5$\to$9).
WF-Act-PC surpasses tuned BP on VGG-5 and ResNet-18 (by $+1.6$ and
$+0.9$\,pp).
On the two deepest CIFAR-100 VGG cells, however, tuned BP catches up: under
symmetric per-method LR tuning the methods tie on VGG-7
($68.76$ vs.\ $69.00$, $-0.24$\,pp) and tuned BP leads on VGG-9
($71.06$ vs.\ $71.90$, $-0.84$\,pp; small but statistically significant).
The depth-scaling property of WF-Act-PC itself is unaffected.
CIFAR-100 uses CutMix/Mixup with learning rate $2{\times}10^{-3}$
(VGG-5/7) or $1{\times}10^{-3}$ (VGG-9, ResNet-18); see supplementary.

\paragraph{Scaling beyond CIFAR.}
To test depth scaling on a harder benchmark, we evaluate on Tiny-ImageNet
(200 classes, $64{\times}64$) with the unchanged WF-Act-PC recipe.
VGG-5 reaches $46.0$\% and ResNet-18 reaches $61.1$\% top-1 (vs.\
per-architecture-tuned BP at $45.0$\% and $59.2$\%), a $+15.1$\,pp gain with
depth and $+1.0$ / $+1.9$\,pp over tuned BP at each depth.
These Tiny-ImageNet runs are \emph{single-seed} (following our
Tiny-ImageNet protocol) and should be read as a scaling data point rather
than a precise benchmark; the full per-architecture BP LR grid is in
Suppl.\ Mat.\ Sect.~H.

\paragraph{Wall-clock time.}
The $T{=}20$ inner loop incurs ${\approx}3{\times}$ overhead relative to
single-pass BP on a single A100 GPU: 22\,s vs.\ 7\,s per epoch on
VGG-5, 57\,s vs.\ 13\,s on ResNet-18.
This cost is inherent to iterative inference; total training for VGG-5 is
${\approx}18$\,min (50 epochs).
The inner loop also affords a clean accuracy--compute trade-off: even
$T{=}1$ (a single Gauss-Seidel sweep, $6.4$\,s/epoch, near BP speed) reaches
$87.64$\% on VGG-5/CIFAR-10, already above iPC ($85.51$\%; cf.\
Table~\ref{tab:ablation}).
Layer-parallel neuromorphic hardware would amortise the inner loop
across cores; a detailed timing breakdown appears in Suppl.\ Mat.\
Sect.~G.

\subsection{Depth Scaling Analysis}
\label{sec:depth_scaling}

Figure~\ref{fig:depth_scaling} visualises the central empirical finding:
WF-Act-PC is the only PC method whose accuracy \emph{increases} with depth
on both datasets.
Positive depth scaling is the expected consequence of correct $\JT$
transport: deeper networks have greater representational capacity, and
exact error routing lets WF-Act-PC exploit it---mirroring the behaviour
of backpropagation.

\begin{figure}[t]
  \centering
  \begin{subfigure}[t]{0.48\linewidth}
    \centering
    \includegraphics[width=\linewidth]{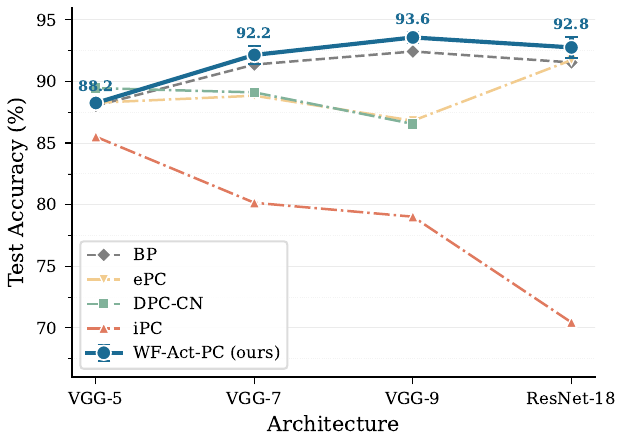}
    \caption{CIFAR-10}
    \label{fig:depth_cifar10}
  \end{subfigure}\hfill
  \begin{subfigure}[t]{0.48\linewidth}
    \centering
    \includegraphics[width=\linewidth]{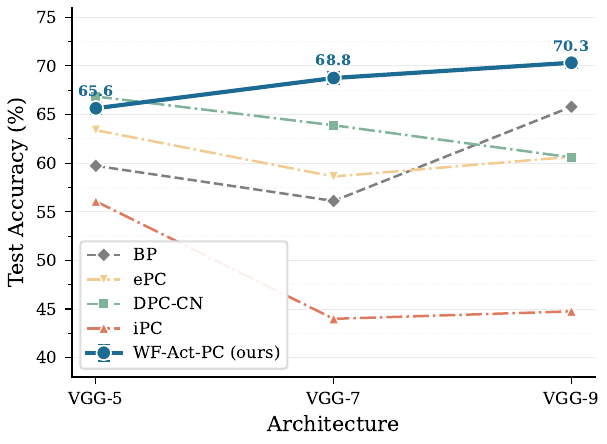}
    \caption{CIFAR-100}
    \label{fig:depth_cifar100}
  \end{subfigure}
  \caption{Depth scaling on CIFAR-10 and CIFAR-100.
  WF-Act-PC (solid blue) is the only PC method whose accuracy
  \emph{increases monotonically} with depth on both datasets.
  Error bars: $\pm$1 std over 5 seeds.}
  \label{fig:depth_scaling}
\end{figure}

\subsection{Ablation Study}
\label{sec:ablation}

Table~\ref{tab:ablation} ablates the key design choices on
VGG-5/CIFAR-10, using the same hyperparameters as
Table~\ref{tab:main} (mean\,$\pm$\,std, 5 seeds, 50 epochs).

\begin{table}[t]
\centering
\caption{Ablation study on VGG-5/CIFAR-10 (mean\,$\pm$\,std over 5 seeds,
same hyperparameters as Table~\ref{tab:main}).}
\label{tab:ablation}
\renewcommand{\arraystretch}{1.15}
\begin{tabular}{llcc}
  \toprule
  \textbf{Group} & \textbf{Variant}
    & \textbf{Acc (\%)} & \textbf{s/epoch} \\
  \midrule
  \multirow{4}{*}{Transport}
    & Random $B$ (FA-style) & $27.42_{\pm0.83}$ & 22.0 \\
    & $\WT$ only (no $\sigp$) & $81.62_{\pm0.47}$ & 20.1 \\
    & $\WT + \sigp$ (\textbf{WF-Act-PC}) & $\mathbf{88.25}_{\pm0.31}$ & 22.1 \\
    & Autograd $\JT$ (\texttt{jax.vjp}) & $89.01_{\pm0.24}$ & --- \\
  \midrule
  \multirow{5}{*}{Inner loop $T$}
    & $T = 1$ & $87.64_{\pm0.38}$ & 6.4 \\
    & $T = 5$ & $88.07_{\pm0.29}$ & 10.8 \\
    & $T = 10$ & $88.29_{\pm0.33}$ & 15.5 \\
    & $T = 20$ (default) & $\mathbf{88.25}_{\pm0.31}$ & 22.1 \\
    & $T = 40$ & $72.01_{\pm0.91}$ & 33.2 \\
  \midrule
  Stability & SNC disabled & $87.40_{\pm0.52}$ & 20.6 \\
  \bottomrule
\end{tabular}
\end{table}

\paragraph{Transport mechanism.}
Random $B$ (FA-style) collapses to $27.42$\%, consistent with the known
failure of FA on deep
CNNs~\cite{bartunov2018assessing,launay2020direct}: structured feedback
is necessary.
$\WT$ alone (without $\sigp$) yields $81.62$\%; adding the $\sigp$
correction recovers $88.25$\% ($+6.6$\,pp), establishing $\sigp$ as the
critical factor in Observation~\ref{thm:wf}.
Importantly, both ablations run the \emph{full} WF-Act-PC recipe (SNC,
top-clip, Cutout, label smoothing), so neither the stabilisation stack nor
the network topology alone explains the result: the
$\mathbf{s}{\odot}\sigp$ correction is essential, not incidental.
Replacing local transport with autograd $\JT$ gives $89.01$\%---a gap of
only $0.8$\,pp, within seed variance.
This narrow gap validates that the local formula faithfully implements
$\JT$; the value of local transport is
\emph{bio-plausibility}---eliminating the autograd
dependency---rather than additional accuracy.

\paragraph{Inner loop iterations.}
Accuracy improves from $T{=}1$ ($87.64$\%) through $T{=}10$
($88.29$\%) and plateaus at $T{=}20$ ($88.25$\%).
$T{=}40$ overshoots: RMSProp's adaptive denominator shrinks over many
iterations, amplifying updates and destabilising convergence.
The effective range is $T{=}10$--$20$, fewer than the $T{=}30$ used by
iPC~\cite{salvatori2024stable,pinchetti2025benchmarking}.
Even $T{=}1$---a single Gauss-Seidel sweep at $6.4$\,s/epoch---achieves
$87.64$\%, surpassing iPC ($85.51$\%) at $3.5{\times}$ lower cost.

\paragraph{Sensitivity to $\tau$ and $T$.}
Sweeping the SNC threshold $\tau \in \{2, 3, 5, \infty\}$ on
VGG-5/CIFAR-10 (single seed) gives $87.9 / 88.3 / 87.9 / 87.4$\%:
accuracy varies by $\le 0.9$\,pp across the range, including $\tau{=}\infty$
(no clipping), so WF-Act-PC is not brittle to this hyperparameter.
Likewise, extending the inner loop to $T{=}30$ holds accuracy
($88.2$\%, single seed); the collapse at $T{=}40$ (Table~\ref{tab:ablation})
is RMSProp overshoot, outside the recommended $T\in[10,20]$ range.
The method is therefore robust over $T\in[5,30]$ and $\tau\in[2,\infty)$.

\paragraph{Spectral norm clipping.}
Disabling SNC drops accuracy by only $0.85$\,pp ($87.40$\% vs.\
$88.25$\%), still exceeding iPC by $+1.9$\,pp.
SNC is a practical stabiliser, not a prerequisite: WF-Act-PC's core
advantage---local $\JT$ transport via $\sigp$---persists without it.
This matters for bio-plausibility, since SNC (power iteration) is the
one non-local component beyond the symmetric synapse.

\paragraph{BatchNorm correction ($s \neq 1$).}
The preceding ablation uses no-BN VGGs where $s{=}1$, exercising only the
$\sigp$ correction.
To validate the full Observation~\ref{thm:wf} including the
BN gain $s = \gamma/\!\sqrt{\mathrm{Var}+\epsilon}$, we add BatchNorm
to VGG-5 and selectively ablate each factor
(Table~\ref{tab:ablation_bn}).
Removing $s$ drops accuracy by $1.6$\,pp ($86.03$\% vs.\ $87.60$\%);
removing $\sigp$ costs $3.1$\,pp; removing both ($\WT$ only) costs
$3.9$\,pp.
Autograd $\JT$ reaches $89.82$\%, indicating that the residual
$2.2$\,pp gap to autograd is attributable to the MaxPool
nearest-neighbour approximation (the only non-exact component).
All gaps are large relative to the seed spread ($\le 0.3$\,pp), so both
correction factors in Observation~\ref{thm:wf} contribute meaningfully
when $s \neq 1$.

\begin{table}[t]
\centering
\caption{BN-VGG-5/CIFAR-10 ablation (mean\,$\pm$\,std over 5 seeds): effect of
each correction factor in Observation~\ref{thm:wf} when BatchNorm is present
($s \neq 1$).}
\label{tab:ablation_bn}
\renewcommand{\arraystretch}{1.15}
\begin{tabular}{lc}
  \toprule
  \textbf{Transport variant} & \textbf{Acc (\%)} \\
  \midrule
  $\WT$ only (no $\sigp$, no $s$)         & $83.74_{\pm0.14}$ \\
  $\WT \cdot \sigp$ (no $s$)              & $86.03_{\pm0.22}$ \\
  $\WT \cdot s$ (no $\sigp$)              & $84.48_{\pm0.20}$ \\
  $\WT \cdot \sigp \cdot s$ (\textbf{full}) & $87.60_{\pm0.23}$ \\
  Autograd $\JT$ (\texttt{jax.vjp})        & $89.82_{\pm0.27}$ \\
  \bottomrule
\end{tabular}
\end{table}

\paragraph{MaxPool unpooling.}
Our default uses nearest-neighbour unpooling---a local
approximation that avoids caching argmax positions
(Remark~\ref{rem:special}).
Comparing against exact argmax unpooling (which scatters each error to
the cached argmax position) on VGG-5/CIFAR-10 (mean\,$\pm$\,std over 5
seeds) yields $88.32_{\pm0.35}$\% vs.\ $89.58_{\pm0.21}$\%---a gap of only
$1.25$\,pp, validating the local approximation.
Switching to exact unpooling closes the small residual gap between local
transport and autograd $\JT$ ($89.01$\%, Table~\ref{tab:ablation}),
confirming that MaxPool, not the $\WT \sigp$ formula, is the dominant
source of the local-vs-autograd gap.

\paragraph{Inner-loop convergence.}
\label{sec:convergence_exp}
Figure~\ref{fig:convergence_gap} tracks the convergence gap $\Delta$
(Eq.~\ref{eq:gap}) during training with $T{=}20$.
$\Delta$ drops to ${\sim}10^{-4}$ within the first few epochs across all
architectures, indicating reliable convergence.
Figure~\ref{fig:training_curves} shows that deeper models consistently
reach higher final accuracy.

\begin{figure}[t]
  \centering
  \begin{subfigure}[t]{0.48\linewidth}
    \centering
    \includegraphics[width=0.78\linewidth]{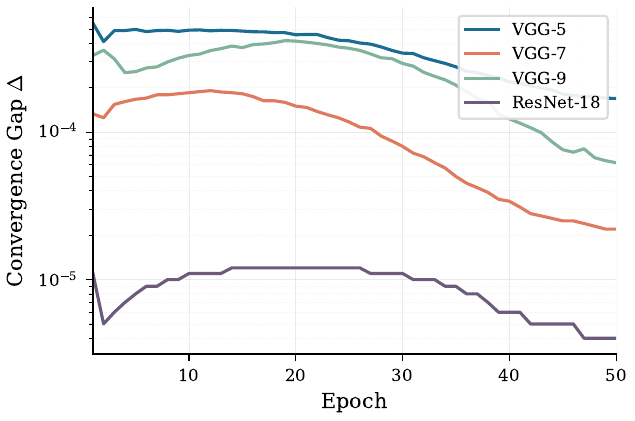}
    \caption{Convergence gap $\Delta$}
    \label{fig:convergence_gap}
  \end{subfigure}\hfill
  \begin{subfigure}[t]{0.48\linewidth}
    \centering
    \includegraphics[width=0.78\linewidth]{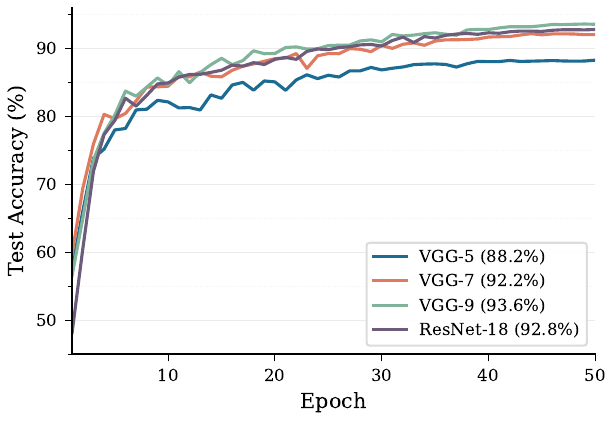}
    \caption{Test accuracy vs.\ epoch}
    \label{fig:training_curves}
  \end{subfigure}
  \caption{Training dynamics on CIFAR-10.
  \textbf{(a)}~Convergence gap $\Delta$ (Eq.~\ref{eq:gap}) drops to
  near-zero for all architectures, confirming $T{=}20$ iterations suffice.
  \textbf{(b)}~Deeper models reach higher final accuracy, consistent with
  positive depth scaling (Table~\ref{tab:main}).}
  \label{fig:training_dynamics}
\end{figure}

\section{Discussion and Conclusion}
\label{sec:conclusion}

\paragraph{Why WF-Act-PC matches or exceeds backpropagation.}
On VGG-7/9/ResNet-18 (CIFAR-10), WF-Act-PC outperforms BP~(tuned) by
$+0.8$/$+1.1$/$+1.2$\,pp.
We tentatively attribute this to the Gauss-Seidel inner loop acting as an
implicit regulariser, iteratively refining activations toward a local energy
minimum that may favour generalisation---a denoising effect absent from
single-pass BP.\footnote{This is a hypothesis to explain the gap, not a
claim established by controlled experiments.}
That even $T{=}1$ surpasses iPC ($87.64$\% vs.\ $85.51$\%,
Table~\ref{tab:ablation}) suggests the advantage originates in the local
$\JT$ transport formula rather than the number of inner-loop iterations.

\paragraph{Relation to ePC.}
As discussed in Sect.~\ref{sec:prelim}, WF-Act-PC shares the error-based
formulation and Gauss-Seidel sweep with ePC~\cite{goemaere2025epc}, paired
with an RMSProp inner loop.
The autograd-vs-local ablation (Table~\ref{tab:ablation}) quantifies the
decomposition: the shared training recipe accounts for the accuracy gains
over iPC, while Observation~\ref{thm:wf} eliminates the autograd
dependency at negligible accuracy cost ($88.25$\% local vs.\ $89.01$\%
autograd, within seed variance).
The observation's contribution is therefore not additional accuracy but
\emph{locality}: WF-Act-PC computes $\JT$ transport from local factors with
no autograd VJP and no performance loss relative to its autograd-dependent
counterpart.

\paragraph{Limitations.}
We restate the assumptions flagged in Sect.~\ref{sec:intro}, as none is
removed by this work.
(i)~\emph{Weight symmetry}: $\LT$ is assumed to mirror the feedforward
operator---shared with all PC methods and BP alike; combining WF-Act-PC with
learned approximate feedback~\cite{akrout2019deep} remains open.
(ii)~\emph{Soft SNC}: power iteration is not synapse-local, though it
operates on the slow learning timescale and disabling it costs only
$0.85$\,pp (Table~\ref{tab:ablation}).
(iii)~\emph{MaxPool}: the nearest-neighbour unpool is a local approximation
(${\approx}1.25$\,pp vs.\ exact argmax), so Eq.~(\ref{eq:wf_identity}) is exact
only for the non-pooling portion of the network.
Beyond these, the $T{=}20$ inner loop adds ${\approx}3{\times}$ wall-clock
overhead (Sect.~\ref{sec:cifar100}), inherent to iterative inference, and
our evaluation follows the PCX benchmark protocol on CIFAR-10/100 with a
Tiny-ImageNet scaling check---ImageNet-1K, deeper ResNets, and attention-based
architectures via frozen LayerNorm (Suppl.\ Mat.\ Sect.~B) are natural next
steps.
Finally, the picture against backpropagation is not uniformly favourable:
on VGG-9/CIFAR-100, symmetrically-tuned BP exceeds WF-Act-PC by $0.84$\,pp
(Sect.~\ref{sec:cifar100}), and the benchmark's own BP references
(BP-CE/BP-SE)---which use a stronger per-architecture recipe---exceed
WF-Act-PC on ResNet-18 (both datasets).

\paragraph{Conclusion.}
We showed that weight feedback combined with two locally computable
corrections---$\sigp(\mathrm{pre\text{-}act})$ and
$\mathbf{s} = \gamma/\sigma_{\mathrm{run}}$---computes the exact Jacobian
transpose for $\mathrm{Act}(\mathrm{Norm}(L(\cdot)))$ layers, up to the
weight-symmetry, soft-SNC, and MaxPool caveats above.
Neither the FA nor the PC community had put these pieces together.
Substituting this identity into PC yields WF-Act-PC, which surpasses iPC (by up to
$+22.3$\,pp) and reaches BP-level accuracy on deep convolutional
architectures under per-architecture tuning---the only PC method whose
accuracy improves with depth.
Weight feedback did not fail because $\WT$ was used; it failed because
$\sigp$ and $\gamma/\sigma$ were left out.

\subsubsection*{Acknowledgements.}
This work was supported in part by the Canada CIFAR AI Chairs Program, the
Alberta Machine Intelligence Institute (Amii), and the Natural Sciences and
Engineering Research Council of Canada (NSERC).

\subsubsection*{Disclosure of Interests.}
The authors declare that they have no competing interests.


\clearpage
\setcounter{section}{0}
\renewcommand{\thesection}{\Alph{section}}
\renewcommand{\thesubsection}{\thesection.\arabic{subsection}}
\begin{center}
  {\Large\bfseries Supplementary Material}\\[3pt]
  {\large Weight Feedback Computes the Jacobian Transpose Locally\\
   in Modern Deep Networks}
\end{center}
\bigskip

\renewcommand{\thesection}{\Alph{section}}

\section{Full Proof of Observation~1 with Dimension Tracking}
\label{app:proof}

We provide an explicit dimension-tracking proof for the Conv+BN+Act case, the
standard block in VGG and ResNet.

\paragraph{Setup.}
Let $W \in \mathbb{R}^{C_{\mathrm{out}} \times C_{\mathrm{in}} \times k \times k}$
be a convolutional kernel and $x \in \mathbb{R}^{C_{\mathrm{in}} \times H \times
S}$ be the input feature map ($S$ = spatial width).  The convolution $z = W \ast x$ yields
$z \in \mathbb{R}^{C_{\mathrm{out}} \times H' \times S'}$.  BatchNorm operates
channel-wise with frozen parameters $\{\gamma_c, \beta_c, \mu^{\mathrm{run}}_c,
\sigma^{\mathrm{run}}_c\}$; the activation $\mathrm{Act}$ is applied elementwise to
produce $h \in \mathbb{R}^{C_{\mathrm{out}} \times H' \times S'}$.

\paragraph{Step 1 (Linear operator).}
The adjoint of $L: x \mapsto W \ast x$ is $\LT: v \mapsto W \ast_{\mathrm{T}} v$
(transposed convolution, \texttt{conv\_transpose2d}), mapping
$\mathbb{R}^{C_{\mathrm{out}} \times H' \times S'} \to
\mathbb{R}^{C_{\mathrm{in}} \times H \times S}$.  This follows from
$\langle W \ast x,\; v \rangle = \langle x,\; W \ast_{\mathrm{T}} v \rangle$.

\paragraph{Step 2 (Frozen BatchNorm).}
With frozen statistics, for each channel $c$ and spatial position $(i, j)$:
\[
  \mathrm{BN}(z)_{c,i,j} = \frac{\gamma_c}{\sigma^{\mathrm{run}}_c}(z_{c,i,j}
  - \mu^{\mathrm{run}}_c) + \beta_c = s_c \cdot z_{c,i,j} + b_c,
  \quad s_c = \frac{\gamma_c}{\sigma^{\mathrm{run}}_c}.
\]
The Jacobian is $J_{\mathrm{BN}} = \mathrm{diag}(\mathbf{s})$ with $\mathbf{s}
\in \mathbb{R}^{C_{\mathrm{out}}}$ broadcast spatially, giving
$J_{\mathrm{BN}}^{\top} \cdot u = \mathbf{s} \odot u$ for $u \in
\mathbb{R}^{C_{\mathrm{out}} \times H' \times W'}$.

\paragraph{Step 3 (Pointwise activation).}
Let $p = \mathrm{BN}(z) \in \mathbb{R}^{C_{\mathrm{out}} \times H' \times W'}$.
Since $\mathrm{Act}$ is elementwise, $J_{\mathrm{Act}} = \mathrm{diag}(\sigp(p))$
and $J_{\mathrm{Act}}^{\top} \cdot v = \sigp(p) \odot v$.

\paragraph{Composition and dimension check.}
For $v \in \mathbb{R}^{C_{\mathrm{out}} \times H' \times W'}$:
\begin{align*}
  \JT \cdot v
  &= J_L^{\top} \cdot J_{\mathrm{BN}}^{\top} \cdot J_{\mathrm{Act}}^{\top} \cdot v
   = J_L^{\top} \cdot J_{\mathrm{BN}}^{\top} \cdot (\sigp(p) \odot v) \\
  &= J_L^{\top} \cdot (\mathbf{s} \odot \sigp(p) \odot v)
   = W \ast_{\mathrm{T}} (\mathbf{s} \odot \sigp(p) \odot v).
\end{align*}
The intermediate tensor $\mathbf{s} \odot \sigp(p) \odot v \in
\mathbb{R}^{C_{\mathrm{out}} \times H' \times S'}$ and
$W \ast_{\mathrm{T}} (\cdot) \in \mathbb{R}^{C_{\mathrm{in}} \times H \times
S}$, recovering the correct input-space dimension. $\square$

\section{Instantiations and Extension to Attention Layers}
\label{app:instantiations}

\subsection{Normalization Variants}

The identity $\JT \cdot v = \LT(\mathbf{s} \odot \sigp(\mathrm{pre\text{-}act})
\odot v)$ applies to any linear+frozen-norm+pointwise-activation layer.  Only the
scale $\mathbf{s}$ varies across architectures.

\begin{table}[htbp]
\centering
\caption{Per-channel scale $\mathbf{s}$ for common normalization choices.  The
transport formula is identical in all cases.}
\label{tab:instantiations}
\renewcommand{\arraystretch}{1.15}
\begin{tabular}{p{4.8cm}ll}
  \toprule
  \textbf{Layer type} & \textbf{Scale $\mathbf{s}$} & \textbf{Notes} \\
  \midrule
  Conv + BatchNorm (frozen) + Act
    & $\gamma_c/\sigma^{\mathrm{run}}_c$ & Per-channel; standard CNN \\
  Linear + BatchNorm (frozen) + Act
    & $\gamma_c/\sigma^{\mathrm{run}}_c$ & MLP with BN \\
  Conv + LayerNorm (frozen) + Act
    & $\gamma/\sigma_{\mathrm{LN}}$      & Over all non-batch dims \\
  Conv + GroupNorm (frozen) + Act
    & $\gamma_g/\sigma_G$                & Per-group \\
  Conv + Act (no norm)
    & $\mathbf{1}$                       & Exact; no special case needed \\
  Linear (no act, no norm)
    & $\mathbf{1}$                       & $\JT = \WT$ trivially \\
  \bottomrule
\end{tabular}
\end{table}

\subsection{Extension to Attention-Based Architectures}
\label{app:attention}

Transformer blocks use Multi-Head Self-Attention (MHSA) followed by an MLP,
each wrapped with LayerNorm:
\[
  z = x + \mathrm{MHSA}(\mathrm{LN}(x)), \qquad
  h = z + \mathrm{MLP}(\mathrm{LN}(z)).
\]
Observation~1 applies directly to the MLP sub-block:
$f_{\mathrm{MLP}}(u) = \mathrm{Act}(\mathrm{LN}(W_1 u))$, giving
$\JT v = W_1^{\top}(\mathbf{s}_{\mathrm{LN}} \odot \sigp \odot v)$
with $\mathbf{s}_{\mathrm{LN}} = \gamma / \sigma_{\mathrm{LN}}$.

The MHSA sub-block is more involved.  For a single attention head with
queries $Q = W_Q x$, keys $K = W_K x$, values $V = W_V x$, and attention
$A = \mathrm{softmax}(QK^{\top}/\sqrt{d})$, the output is $AV$.
The Jacobian of attention w.r.t.\ $x$ is \emph{not} a simple diagonal---it
involves $\partial A / \partial x$, which is data-dependent and dense.

\paragraph{Frozen-attention approximation.}
If the attention matrix $A$ is frozen during the PC inner loop (analogous to
freezing BatchNorm statistics), then
$J_{\mathrm{MHSA}} \approx A \cdot J_V = A \cdot W_V$, giving
$J_{\mathrm{MHSA}}^{\top} v \approx W_V^{\top} A^{\top} v$---again a product
of locally available factors ($W_V^{\top}$ is the symmetric synapse of the
value projection, and $A$ is cached from the forward pass).

This approximation is exact when $\partial A / \partial x = 0$, which holds
when attention weights are treated as fixed during inference---a natural
analogue of frozen normalization statistics.  Whether this approximation
suffices for competitive training accuracy on vision transformers is an
empirical question we leave for future work, but the theoretical framework
of Observation~1 extends naturally to this setting.

\paragraph{Practical considerations.}
For ViT-style architectures, the MLP sub-block accounts for $\approx$67\%
of the parameters (the two linear layers $W_1, W_2$), where Observation~1
applies exactly.  The attention sub-block requires the frozen-attention
approximation above.  Hybrid architectures (e.g., convolutional stems with
transformer blocks) can apply exact local transport in convolutional layers
and the frozen-attention approximation in transformer layers.

\section{CIFAR-100 Experimental Details}
\label{app:cifar100}

\subsection{Hyperparameter Configuration}

CIFAR-100 (100 classes) presents additional challenges compared to CIFAR-10:
the output error vector has larger norm (100-dimensional softmax vs.\
10-dimensional), which stresses inner-loop stability, and finer-grained
classes require stronger augmentation to prevent overfitting.
Table~\ref{tab:cifar100_hparams} lists the full hyperparameter configuration.

\begin{table}[htbp]
\centering
\caption{CIFAR-100 hyperparameters for WF-Act-PC.  Changes relative to the
CIFAR-10 configuration (Table~1 footnote in the main paper) are
\textbf{bolded}.}
\label{tab:cifar100_hparams}
\renewcommand{\arraystretch}{1.15}
\begin{tabular}{lcc}
  \toprule
  \textbf{Hyperparameter} & \textbf{VGG-5/7} & \textbf{VGG-9 / ResNet-18} \\
  \midrule
  Optimizer         & AdamW               & AdamW \\
  Learning rate $\eta_w$
    & $\mathbf{2 \times 10^{-3}}$ & $\mathbf{1 \times 10^{-3}}$ \\
  Weight decay      & $1 \times 10^{-4}$  & $1 \times 10^{-4}$ \\
  LR schedule       & Warmup-cosine (2 ep) & Warmup-cosine (2 ep) \\
  Batch size        & 128                 & 128 \\
  Epochs            & 50                  & 50 \\
  \midrule
  Inner-loop steps $T$ & 20               & 20 \\
  Inner-loop optimizer & RMSProp           & RMSProp \\
  Inference LR $\eta_{\varepsilon}$ & 0.1           & 0.1 \\
  RMSProp decay $\alpha_{\mathrm{rms}}$ & 0.2 & 0.2 \\
  Top-layer error clip & 5.0              & 5.0 \\
  Gradient clip (norm) & 1.0              & 1.0 \\
  \midrule
  Soft SNC $\tau$   & 3.0                 & 3.0 \\
  SNC power iterations & 20               & 20 \\
  \midrule
  \textbf{Label smoothing}
    & $\mathbf{0.1}$       & $\mathbf{0.1}$ \\
  \textbf{CutMix prob.}
    & $\mathbf{0.5}$       & $\mathbf{0.5}$ \\
  \textbf{Mixup $\alpha$}
    & $\mathbf{0.2}$       & $\mathbf{0.2}$ \\
  \textbf{Cutout patch size}
    & $\mathbf{10}$        & $\mathbf{10}$ \\
  RandomCrop(32, pad=4) & Yes             & Yes \\
  RandomHorizontalFlip  & Yes             & Yes \\
  \bottomrule
\end{tabular}
\end{table}

\subsection{Key Differences from CIFAR-10}

Four hyperparameters differ from the CIFAR-10 configuration:

\begin{enumerate}[leftmargin=*,topsep=2pt,itemsep=2pt]
\item \textbf{Learning rate.}
  VGG-5/7 use $\eta_w = 2 \times 10^{-3}$ (vs.\ $7 \times 10^{-4}$ for
  CIFAR-10); VGG-9 and ResNet-18 use $\eta_w = 1 \times 10^{-3}$.
  The higher learning rate compensates for the larger output dimension:
  with 100 classes, the cross-entropy gradient per class is $\sim$10$\times$
  smaller than with 10 classes, so a higher weight LR is needed to maintain
  the same effective update magnitude.
  Deeper models (VGG-9, ResNet-18) use a lower LR for stability.

\item \textbf{Label smoothing.}
  Increased from $0.05$ to $0.1$ to provide stronger regularization for the
  finer-grained 100-class task.

\item \textbf{CutMix / Mixup.}
  Added CutMix (probability 0.5) and Mixup ($\alpha = 0.2$) on top of
  Cutout.  These augmentations create mixed training targets, which has been
  shown to improve calibration and generalization on fine-grained
  tasks~\cite{yun2019cutmix,zhang2018mixup}.

\item \textbf{Cutout patch size.}
  Increased from $8$ to $10$ pixels.
\end{enumerate}

\noindent
All other hyperparameters---inner-loop configuration ($T$, $\eta_{\varepsilon}$,
$\alpha_{\mathrm{rms}}$), Soft SNC ($\tau$, power iterations), optimizer
betas, and gradient clipping---are identical to CIFAR-10.
The BP~(tuned) baseline uses the same augmentation recipe (CutMix, Mixup,
Cutout~10, label smoothing~0.1) but replaces the PC inner loop with
standard backpropagation, with the learning rate tuned per architecture
(grid in Sect.~\ref{app:tuned}).

\subsection{CIFAR-100 Results Discussion}

Table~1 (main paper) reports WF-Act-PC against per-architecture-tuned BP on
CIFAR-100:
\begin{center}
\renewcommand{\arraystretch}{1.1}
\begin{tabular}{lcccc}
  \toprule
  & \textbf{VGG-5} & \textbf{VGG-7} & \textbf{VGG-9} & \textbf{ResNet-18} \\
  \midrule
  WF-Act-PC      & 65.64 & 68.75 & 70.32 & 71.07 \\
  BP (tuned)     & 64.09 & 69.00 & 71.90 & 70.22 \\
  \midrule
  $\Delta$ (WF $-$ BP) & $+1.55$ & $-0.25$ & $-1.58$ & $+0.85$ \\
  \bottomrule
\end{tabular}
\end{center}

\noindent
Under \emph{symmetric} per-method LR tuning, WF-Act-PC reaches
$68.76_{\pm0.31}$ (VGG-7) and $71.06_{\pm0.18}$ (VGG-9), narrowing the gap to
tuned BP to $-0.24$ and $-0.84$\,pp respectively (Sect.~\ref{app:tuned}).
Three observations:

\paragraph{1. Positive depth scaling persists.}
WF-Act-PC improves from 65.64\% (VGG-5) to 71.07\% (ResNet-18), a
$+5.4$\,pp gain with depth.  In contrast, iPC degrades from 56.07\% to
29.45\% ($-26.6$\,pp), and DPC-CN degrades from 66.85\% to 60.59\%
($-6.3$\,pp).  Only ePC shows partial depth robustness (63.39\% to
69.47\%), but it relies on autograd $\JT$ and still underperforms
WF-Act-PC at every depth.

\paragraph{2. Comparison to tuned BP is mixed but favourable on the extremes.}
WF-Act-PC exceeds tuned BP on VGG-5 ($+1.55$\,pp) and ResNet-18
($+0.85$\,pp).  On the deeper VGG cells, tuned BP catches up: under
symmetric tuning the methods tie on VGG-7 ($-0.24$\,pp, within noise) and
tuned BP leads on VGG-9 ($-0.84$\,pp; small but statistically significant,
Welch $t \approx 5$, $p < 0.001$).  We report this rather than select
baselines.  We note, but do not rely on, the possibility that the
Gauss-Seidel inner loop acts as an implicit regulariser more helpful on
easier/shallower settings; this is a preliminary hypothesis, not a
controlled finding.

\paragraph{3. Comparison to the benchmark's BP references.}
WF-Act-PC trails the squared-error BP-SE only at the extremes---VGG-5
($65.64$ vs.\ $66.28$, $-0.6$\,pp) and ResNet-18 ($71.07$ vs.\ $71.89$,
$-0.8$\,pp)---while \emph{exceeding} it on the intermediate depths VGG-7
($68.75$ vs.\ $65.36$) and VGG-9 ($70.32$ vs.\ $65.51$), where BP-SE
plateaus but WF-Act-PC keeps scaling.  Against the cross-entropy BP-CE,
WF-Act-PC leads at every depth except ResNet-18 ($71.07$ vs.\ $72.32$).
The remaining gaps at the extremes are within roughly a point, and the
positive depth-scaling trend is unaffected.

\section{Complete Hyperparameter Table}
\label{app:hparams}

For reproducibility, Table~\ref{tab:full_hparams} lists the complete
hyperparameter configuration for all experiments reported in the main
paper.

\begin{table}[htbp]
\centering
\caption{Full hyperparameter configuration for all WF-Act-PC experiments.
All CIFAR-10 models share a single configuration (no per-model tuning).
CIFAR-100 differs only in the bolded entries.}
\label{tab:full_hparams}
\renewcommand{\arraystretch}{1.15}
\resizebox{\textwidth}{!}{%
\begin{tabular}{lcc}
  \toprule
  \textbf{Hyperparameter} & \textbf{CIFAR-10 (all models)} & \textbf{CIFAR-100} \\
  \midrule
  \multicolumn{3}{l}{\textit{Weight learning}} \\
  \quad Optimizer          & AdamW ($\beta_1{=}0.9, \beta_2{=}0.999$)
                           & AdamW ($\beta_1{=}0.9, \beta_2{=}0.999$) \\
  \quad Learning rate $\eta_w$
    & $7 \times 10^{-4}$
    & $\mathbf{2 \times 10^{-3}}$ (VGG-5/7) /
      $\mathbf{1 \times 10^{-3}}$ (VGG-9, RN-18) \\
  \quad Weight decay       & $1 \times 10^{-4}$ & $1 \times 10^{-4}$ \\
  \quad LR schedule        & Warmup-cosine: 2 ep warmup, decay to $10^{-6}$
                           & Same \\
  \quad Gradient clip (max norm) & 1.0 & 1.0 \\
  \midrule
  \multicolumn{3}{l}{\textit{PC inner loop (inference)}} \\
  \quad Inner-loop steps $T$ & 20 & 20 \\
  \quad Optimizer            & RMSProp & RMSProp \\
  \quad Inference LR $\eta_{\varepsilon}$ & 0.1 & 0.1 \\
  \quad RMSProp decay $\alpha_{\mathrm{rms}}$ & 0.2 & 0.2 \\
  \quad Top-layer error clip $c$ & 5.0 & 5.0 \\
  \quad Update order         & Gauss-Seidel (top $\to$ bottom) & Same \\
  \midrule
  \multicolumn{3}{l}{\textit{Soft Spectral Norm Clipping}} \\
  \quad Threshold $\tau$     & 3.0 & 3.0 \\
  \quad Power iterations     & 20  & 20 \\
  \midrule
  \multicolumn{3}{l}{\textit{Data and augmentation}} \\
  \quad Batch size           & 128 & 128 \\
  \quad Epochs               & 50  & 50 \\
  \quad RandomCrop(32, pad=4) & Yes & Yes \\
  \quad RandomHorizontalFlip(0.5) & Yes & Yes \\
  \quad Cutout patch size    & 8   & $\mathbf{10}$ \\
  \quad Label smoothing      & 0.05 & $\mathbf{0.1}$ \\
  \quad CutMix (prob.)       & ---  & $\mathbf{0.5}$ \\
  \quad Mixup ($\alpha$)     & ---  & $\mathbf{0.2}$ \\
  \midrule
  \multicolumn{3}{l}{\textit{Loss}} \\
  \quad Output loss $\mathcal{L}$ & Cross-entropy & Cross-entropy \\
  \quad Inner energy         & $\frac{1}{2}\sum_l \|\eps_l\|^2$ & Same \\
  \midrule
  \multicolumn{3}{l}{\textit{Implementation}} \\
  \quad Framework            & JAX 0.4+ (CUDA 12) & Same \\
  \quad Matmul precision     & float32 & float32 \\
  \quad Hardware             & 1$\times$ NVIDIA A100 (40 GB) & Same \\
  \quad Seeds                & 5 (42, 43, 44, 45, 46) & Same \\
  \bottomrule
\end{tabular}}
\end{table}

\section{Architecture Details}
\label{app:arch}

Table~\ref{tab:arch} specifies the exact architecture configurations.

\begin{table}[htbp]
\centering
\caption{Architecture configurations for all models.  VGG variants follow
the PCX benchmark~\cite{pinchetti2025benchmarking} (Conv+GELU+MaxPool,
\emph{no} BatchNorm).  ResNet-18 uses Conv+GELU with residual connections,
also \emph{without} BatchNorm.}
\label{tab:arch}
\renewcommand{\arraystretch}{1.15}
\resizebox{\textwidth}{!}{%
\begin{tabular}{lp{9cm}rc}
  \toprule
  \textbf{Model} & \textbf{Layer configuration} & \textbf{Params}
    & \textbf{BN?} \\
  \midrule
  VGG-5
    & Conv(3$\to$128)+GELU+Pool,
      Conv(128$\to$256)+GELU+Pool,
      Conv(256$\to$512)+GELU+Pool,
      Conv(512$\to$512)+GELU+Pool,
      FC(2048$\to C$)
    & 3.9M & No \\
  \midrule
  VGG-7
    & \{Conv(3$\to$128)+GELU, Conv(128$\to$128)+GELU+Pool\},
      \{Conv(128$\to$256)+GELU, Conv(256$\to$256)+GELU+Pool\},
      \{Conv(256$\to$512)+GELU, Conv(512$\to$512)+GELU+Pool\},
      FC(8192$\to C$)
    & 4.7M & No \\
  \midrule
  VGG-9
    & \{Conv(3$\to$128)+GELU, Conv(128$\to$128)+GELU+Pool\},
      \{Conv(128$\to$256)+GELU, Conv(256$\to$256)+GELU+Pool\},
      \{Conv(256$\to$512)+GELU, Conv(512$\to$512)+GELU+Pool\},
      \{Conv(512$\to$512)+GELU, Conv(512$\to$512)+GELU+Pool\},
      FC(2048$\to C$)
    & 9.3M & No \\
  \midrule
  ResNet-18
    & InitConv(3$\to$64, stride=1), \\
    & Stage 1: 2$\times$ BasicBlock(64$\to$64), \\
    & Stage 2: BasicBlock(64$\to$128, stride=2) + BasicBlock(128$\to$128), \\
    & Stage 3: BasicBlock(128$\to$256, stride=2) + BasicBlock(256$\to$256), \\
    & Stage 4: BasicBlock(256$\to$512, stride=2) + BasicBlock(512$\to$512), \\
    & GlobalAvgPool, FC(512$\to C$)
    & 11.2M & No \\
  \bottomrule
\end{tabular}}
\begin{flushleft}
  \footnotesize
  $C = 10$ for CIFAR-10, $C = 100$ for CIFAR-100.
  All convolutions use $3{\times}3$ kernels with padding 1.
  MaxPool uses $2{\times}2$ windows with stride 2.
  BasicBlock $=$ GELU(Conv\,$+$\,GELU\,$+$\,Conv\,$+$\,skip), no BatchNorm,
  with a $1{\times}1$ downsample convolution when stride $> 1$.
\end{flushleft}
\end{table}

\section{MaxPool Approximation Details}
\label{app:maxpool}

As noted in Remark~2 of the main paper, MaxPool is the one operation in VGG
architectures whose exact $\JT$ is not trivially local: the true Jacobian
transpose routes the gradient only through the argmax position within each
pool window, requiring cached argmax indices from the forward pass.

Our implementation replaces the exact MaxPool $\JT$ with nearest-neighbour
upsampling: each error value is repeated uniformly over the $2{\times}2$
pool window.  This avoids caching argmax positions and depends only on the
pool window size (a fixed structural constant).

A direct ablation on VGG-5/CIFAR-10 (mean\,$\pm$\,std over 5 seeds; same
hyperparameters as Table~2 in the main paper) quantifies the cost:
\begin{itemize}[nosep,leftmargin=*]
  \item \textbf{Nearest-neighbour unpooling} (default, local):
    $88.32_{\pm0.35}$\%.
  \item \textbf{Exact argmax unpooling} (caches pool indices):
    $89.58_{\pm0.21}$\%.
\end{itemize}
The gap of $1.25$\,pp validates the local approximation.
Switching to exact unpooling closes the small residual gap between local
transport ($88.25$\%) and autograd $\JT$ ($89.01$\%) in Table~2, confirming
that MaxPool---not the $\WT \sigp$ formula for linear+activation
layers---is the dominant source of the local-vs-autograd gap.

\section{Wall-Clock Timing Breakdown}
\label{app:timing}

Table~\ref{tab:timing_breakdown} breaks down the per-step wall-clock time
for VGG-5 on CIFAR-10 (single A100 GPU).

\begin{table}[htbp]
\centering
\caption{Wall-clock timing breakdown for one training step on VGG-5/CIFAR-10
(batch size 128, single A100).}
\label{tab:timing_breakdown}
\renewcommand{\arraystretch}{1.1}
\begin{tabular}{lrl}
  \toprule
  \textbf{Phase} & \textbf{Time (ms)} & \textbf{Notes} \\
  \midrule
  Forward pass + caching         & 3.2 & Cache $\mathrm{pre\text{-}act}$,
    $\mathbf{s}$ \\
  Inner loop ($T{=}20$ sweeps)   & 42.1 & Gauss-Seidel + RMSProp \\
  Weight update (AdamW)          & 5.8 & Local gradients + optimizer step \\
  Soft SNC (power iteration)     & 5.4 & 20 power iterations per layer \\
  Data loading + misc.           & 0.4 & Overlapped with compute \\
  \midrule
  \textbf{Total}                 & \textbf{56.9} & $\approx$22\,s/epoch \\
  \midrule
  BP equivalent (forward + backward) & 18.2 & $\approx$7.1\,s/epoch \\
  \bottomrule
\end{tabular}
\end{table}

\noindent
The inner loop dominates ($\approx$74\% of total time).  This cost is
inherent to iterative PC inference and scales linearly with $T$.
Notably, each Gauss-Seidel sweep is embarrassingly parallel across layers
in the error-based formulation (all $\eps_l$ receive gradient signals from
iteration 1), suggesting that dedicated hardware with layer-parallel
execution could reduce the inner-loop latency significantly.

\section{Tuned-BP Learning-Rate Grids and Tiny-ImageNet}
\label{app:tuned}

This section gives the per-architecture learning-rate search behind the
``BP (tuned)'' column of Table~1 and the Tiny-ImageNet results of
Sect.~5.3, and documents the \emph{symmetric} tuning of WF-Act-PC on the two
CIFAR-100 cells where tuned BP is competitive.

\paragraph{Protocol.}
BP (tuned) uses the same augmentation recipe as WF-Act-PC (Cutout / CutMix /
Mixup, label smoothing, cosine warmup-decay, weight decay $10^{-4}$) and
searches the learning rate over
$\{3{\times}10^{-4}, 7{\times}10^{-4}, 1{\times}10^{-3}, 3{\times}10^{-3}\}$.
Because WF-Act-PC and BP have markedly different optimal learning rates
(e.g.\ on VGG-7/CIFAR-100, WF-Act-PC is best near $2{\times}10^{-3}$ while BP
is best near $3{\times}10^{-4}$), comparing each method at its own best LR is
the fair protocol; comparing both at a single shared LR flatters whichever
method that LR happens to suit.

\begin{table}[htbp]
\centering
\caption{BP (tuned) learning-rate grid (test accuracy \%, single seed unless
noted). \textbf{Best} per row is the value reported as BP (tuned) in Table~1;
tie-risk cells are then re-run for 5 seeds (rightmost column).}
\label{tab:tuned_grid}
\renewcommand{\arraystretch}{1.15}
\begin{tabular}{llcccccl}
  \toprule
  \textbf{Dataset} & \textbf{Arch}
    & $3{\times}10^{-4}$ & $7{\times}10^{-4}$
    & $1{\times}10^{-3}$ & $3{\times}10^{-3}$
    & \textbf{Best} & \textbf{5-seed} \\
  \midrule
  \multirow{3}{*}{CIFAR-10}
    & VGG-7    & 91.46 & 91.37 & 91.09 & 85.46 & 91.46 & $91.39_{\pm0.05}$ \\
    & VGG-9    & 92.15 & 92.43 & 92.21 & 88.60 & 92.43 & $92.43_{\pm0.29}$ \\
    & ResNet-18 & 91.02 & 91.54 & 91.08 & 85.97 & 91.54 & $91.54_{\pm0.36}$ \\
  \midrule
  \multirow{4}{*}{CIFAR-100}
    & VGG-5    & 64.12 & ---   & 63.23 & 56.32 & 64.12 & $64.09_{\pm0.14}$ \\
    & VGG-7    & 69.03 & 56.11 & 66.02 & 50.92 & 69.03 & $69.00_{\pm0.30}$ \\
    & VGG-9    & 71.93 & ---   & 66.22 & 41.79 & 71.93 & $71.90_{\pm0.32}$ \\
    & ResNet-18 & 70.10 & ---  & 66.76 & 48.06 & 70.10 & $70.22_{\pm0.27}$ \\
  \bottomrule
\end{tabular}
\end{table}

\paragraph{Symmetric tuning of WF-Act-PC (CIFAR-100, VGG-7/9).}
The main paper uses $\eta_w \in \{2,1\}{\times}10^{-3}$ for WF-Act-PC on
CIFAR-100.  To match the per-method tuning applied to BP, we additionally
sweep WF-Act-PC on the two cells where tuned BP is competitive
(VGG-7, VGG-9):

\begin{center}
\renewcommand{\arraystretch}{1.1}
\begin{tabular}{lcccc}
  \toprule
  \textbf{Cell} & $3{\times}10^{-4}$ & $5{\times}10^{-4}$
    & $7{\times}10^{-4}$ & \textbf{5-seed @\,best} \\
  \midrule
  VGG-7 / CIFAR-100 & 67.93 & 68.68 & 68.89 & $68.76_{\pm0.31}$ \\
  VGG-9 / CIFAR-100 & 69.90 & 71.20 & 71.24 & $71.06_{\pm0.18}$ \\
  \bottomrule
\end{tabular}
\end{center}

\noindent
At its own best LR ($7{\times}10^{-4}$), WF-Act-PC reaches
$68.76_{\pm0.31}$ (VGG-7) and $71.06_{\pm0.18}$ (VGG-9).  Against tuned BP
($69.00_{\pm0.30}$ and $71.90_{\pm0.32}$), this is a statistical tie on
VGG-7 ($-0.24$\,pp) and a small but significant BP lead on VGG-9
($-0.84$\,pp, Welch $t \approx 5.2$, $p < 0.001$).  We report these honestly;
WF-Act-PC's monotonic depth-scaling property is independent of any BP tuning.

\paragraph{Tiny-ImageNet.}
Tiny-ImageNet (200 classes, $64{\times}64$) uses the unchanged WF-Act-PC
recipe (single seed, following our protocol of one long Tiny-ImageNet run
per architecture); BP is tuned per architecture as above.
Table~\ref{tab:tin_grid} gives the numbers behind the Tiny-ImageNet results
reported in Sect.~5.3 of the main paper.

\begin{table}[htbp]
\centering
\caption{Tiny-ImageNet top-1 accuracy (\%), 50 epochs. WF-Act-PC single seed;
BP (tuned) best of the LR grid (VGG-5: 5 seeds at the best LR; ResNet-18:
single seed per grid point).}
\label{tab:tin_grid}
\renewcommand{\arraystretch}{1.15}
\begin{tabular}{lccccl}
  \toprule
  \textbf{Arch} & \textbf{WF-Act-PC}
    & \multicolumn{3}{c}{\textbf{BP (tuned) LR grid}} & \textbf{BP best} \\
  \cmidrule(lr){3-5}
   & & $3{\times}10^{-4}$ & $7{\times}10^{-4}$ & $1{\times}10^{-3}$ & \\
  \midrule
  VGG-5     & $\mathbf{46.0}$ & 44.5 & 43.1 & 41.7 & $45.04_{\pm0.40}$ \\
  ResNet-18 & $\mathbf{61.1}$ & 59.2 & 55.7 & 52.1 & 59.18 \\
  \bottomrule
\end{tabular}
\end{table}

\noindent
WF-Act-PC exceeds tuned BP at both depths ($+1.0$ / $+1.9$\,pp) and improves
by $+15.1$\,pp from VGG-5 to ResNet-18, confirming positive depth scaling on
a benchmark harder than CIFAR.  These runs are single-seed and should be read
as a scaling data point rather than a precise benchmark.

\bibliographystyle{splncs04}
\bibliography{egbib}

\end{document}